\documentclass[a4paper]{article}

\usepackage[all]{xy}\usepackage[latin1]{inputenc}        
\usepackage[dvips]{graphics,graphicx}
\usepackage{amsfonts,amssymb,amsmath,color,mathrsfs, amstext}
\usepackage{amsbsy, amsopn, amscd, amsxtra, amsthm,authblk, enumerate}
\usepackage{upref}
\usepackage[colorlinks,
            linkcolor=red,
            anchorcolor=red,
            citecolor=red
            ]{hyperref}

\usepackage{geometry}
\usepackage{float}

\usepackage{yhmath}

\usepackage[normalem]{ulem}

\numberwithin{equation}{section}

\def\e{\epsilon}

\def\E{\mathbb{E}}
\def\R{\mathbb{R}}
\def\cS{\mathcal{S}}
\def\cB{\mathcal{B}}

\newcommand{\cL}{\mathcal{L}}
\DeclareMathOperator{\var}{var}

\newtheorem{theorem}{Theorem}
\newtheorem{lemma}{Lemma}
\newtheorem{assumption}{Assumption}
\newtheorem{proposition}{Proposition}
\newtheorem{remark}{Remark}

\usepackage{authblk}

\begin{document}

\title{On uniform-in-time diffusion approximation for stochastic gradient descent}
\date{}
\author[a,b,c]{Lei Li}
\author[a,b]{Yuliang Wang}
\affil[a]{School of Mathematical Sciences, Shanghai Jiao Tong University, Shanghai, 200240, P.R.China.}
\affil[b]{Institute of Natural Sciences, MOE-LSC, Shanghai Jiao Tong University, Shanghai, 200240, P.R.China.}
\affil[c]{Qing Yuan Research Institute, Shanghai Jiao Tong University, Shanghai, 200240, P.R.China.}
\maketitle

\begin{abstract}
The diffusion approximation of stochastic gradient descent (SGD) in current literature is only valid on a finite time interval. In this paper, we establish the uniform-in-time diffusion approximation of SGD,  by only assuming that the expected loss is strongly convex  and some other mild conditions, without assuming the convexity of each random loss function. The main technique is to establish the exponential decay rates of the derivatives of the solution to the backward Kolmogorov equation.
The uniform-in-time approximation allows us to study asymptotic behaviors of SGD via the continuous stochastic differential equation (SDE) even when the random objective function $f(\cdot;\xi)$ is not strongly convex.
\end{abstract}

\section{Introduction}\label{intro}

Stochastic gradient descent (SGD), as a stochastic approximation for the gradient descent, is a simple but powerful optimization method,  where the objective function is often the average of a family of functions. With the ``random mini-batch" idea, instead of directly calculating the sum of gradient of the whole family, SGD uses the sum over a small random set to approximate the big summation \cite{robbins1951stochastic,ross1988taguchi}. SGD is widely used for solving large scale data science problem, which has shown amazing performance for large-scale learning tasks due to its computational and statistical efficiency  \cite{bottou2010large,bubeck2015convex,bottou2016optimization}.
Recent decades have witnessed huge and fast progress in SGD-related research \cite{hulililiu2018,li2019stochastic,ankirchner2021approximating,smith2020generalization,smith2021origin}. Several variants of SGD have been proposed to deal with various tasks more efficiently, including combining with momentum and varying step size, etc \cite{daniel2016learning,zeiler2012adadelta,dauphin2015equilibrated}.

The optimization problem suited for SGD is given by $\min_{x\in \R^d} f(x)$, where 
\begin{gather}\label{eq:def_expected_loss}
f(x):=\E f(x; \xi)
\end{gather}  
is the loss/objective function associated with a certain training set, and $d$ is the dimension for the parameter $x$. 
Here, $\xi \sim \nu$ is a random variable/vector for some probability distribution $\nu$. Compared with $f(x)$,  $f(x; \xi)$ is often much easier to handle for each $\xi$. The SGD iteration with constant step size $\eta$ is then
\begin{gather}\label{eq:sgd}
X_{n+1}=X_n-\eta \nabla f(X_n; \xi_n) ,
\end{gather}
where $\xi_n\sim \nu$ are i.i.d.  so that $\xi_n$ is independent of $X_n$. Then, $\{X_n\}$ is a time homogeneous Markov chain. In practice, $\xi_n\sim \nu$ is often implemented by drawing random sets from training sets and using it to give the ``stochastic gradient''.  By \eqref{eq:def_expected_loss}, $\nabla f(x;\xi_n)$ is an unbiased estimation of true gradient $\nabla f$, namely, 
\[
\mathbb{E}\left[\nabla f(x;\xi)\right] = \nabla f(x),\quad \forall x \in \R^d.
\]
 Besides, the uncertainty introduced by the SGD is helpful for escaping sharp minimizers and for possible better generalization behavior \cite{li2019stochastic,lin2018don}.

As an example, consider training a deep neural network using $N\gg 1$ samples. The loss function is given by $f(x)=\frac{1}{N}\sum_{k=1}^N f_k(x)$. The back propagation algorithm is applied to compute $\nabla f_k(x)$, which is not trivial, making computing $\nabla f(x)$ expensive \cite{cao2009neural,li2016tutorial}. To handle this problem, we can pick a random set $\cB\subset \{1, \ldots, N\}$ with $|\cB|=m\ll N$. Then we identify $\cB$ with $\xi$ and let  $f(x; \xi)=\frac{1}{m}\sum_{k\in \cB}f_k(x)$. Computing the gradient of $f(x;\xi)$ is clearly much cheaper.

Now following the Markov property \cite{durrett1999essentials,durrett2019probability},  the function
\begin{gather}\label{eq:Un}
U^n(x)=\cS^n \varphi(x):=\mathbb{E}_x(\varphi(X_n))
\end{gather}
satisfies the equation \cite{hulililiu2018,feng2017}
\begin{gather}\label{eq:weakmaster}
U^{n+1}(x)=\cS U^n(x) :=\mathbb{E}(U^n(x-\eta \nabla f(x; \xi))).
\end{gather}
This means that $\{\cS^n\}$ is in fact a semigroup.
With the semigroup property, one is thus naturally motivated to approximate $U$ with solutions to some appropriate time continuous equation.

One classical method to approximate SGD is the diffusion approximation, and much work has been done by former researchers \cite{hulililiu2018,ankirchner2021approximating,feng2017,litaie2017,feng2019uniform} . Assuming that $f(\cdot,\xi)$ has bounded derivatives $\nabla f$, in any finite time interval, iterates of SGD  are close in the weak sense to the solution of following stochastic differential equation (SDE):
\begin{gather}\label{eq:classicalsde}
\mathrm{d} X_{t}=-\nabla\left[f(X)+\frac{1}{4} \eta |\nabla f(X) |^{2}\right] d t+\sqrt{\eta \Sigma}\, d W,
\end{gather}
where the matrix $\Sigma$ given by
\begin{gather}
\Sigma=\mathbb{E}_{\xi}\left[\left(\nabla f_{\xi}-\nabla f\right) \otimes\left(\nabla f_{\xi}-\nabla f\right)\right]
\end{gather}
is the covariance matrix of the random gradients, and $W$ is the standard Brownian motion in $\mathbb{R}^d$.
Note that the SDE \eqref{eq:classicalsde} approximates SGD in the weak sense with second order accuracy. If we instead use $dX=-\nabla f(X) dt+\sqrt{\eta\Sigma}\,dW$ for any smooth 
positive definite $\Sigma$, then the approximation has first order weak accuracy. This implies that the first order weak approximation only captures the coarse gradient descent feature, and loses much information, especially the fluctuation in the dynamics, and possibly the implicit bias \cite{smith2020generalization,smith2021origin}. Among those choices,  taking $\Sigma=\var(\nabla f(x,\xi))$ captures the most fluctuation in the corresponding SDE  \cite{hulililiu2018}.

The associated backward Kolmogrov equation to \eqref{eq:classicalsde} is given by
\begin{gather}\label{eq:backwardKol}
\frac{\partial u}{\partial t}=-\nabla f \cdot \nabla u+\eta\left(-\frac{1}{4} \nabla |\nabla f |^{2} \cdot \nabla u+\frac{1}{2} \operatorname{Tr}\left(\Sigma \nabla^{2} u\right)\right),
\end{gather}
and $u(x,t)$ with initial value $u(x,0)=\varphi(x)$ has a representation
\begin{gather}
u(x, t)=\mathbb{E}_x\varphi(X_t) ,
\end{gather}
where $\varphi$ is an arbitrary test function with certain regularity. In particular, we can take $\varphi = f$ to study the asymptotic oscillatory and how objective function converges to global minimum if the objective function admits some ``good" properties, like convexity.

  Other related works regarding diffusion approximation for SGD can be found in \cite{hulililiu2018,litaie2017}.
Many other approximation methods are proposed in \cite{ankirchner2021approximating}, where the authors also study approximation for SGD in some finite time iterval $[0,T]$. Their methods include ODE approximation, first order SDE approximation, and second order SDE approximation, with the change of step size taken into account in each method. Though numerous novel insights have been gained from this continuous perspective, it was previously still unclear whether the modified SDEs can really be adopted to study asymptotic behaviors of SGD, since the weak approximation is only valid in a finite time interval.  In \cite{feng2019uniform},  the authors used a truncated formal power expansion of the solution of a Kolmogrov equation arising from diffusion approximation to obtain uniform-in-time analysis. However, the diffusion approximation itself is still not uniformly valid. Besides, in practice, the boundedness assumption on $\nabla f$ is strong, which also motivates us to establish the uniform-in-time diffusion approximation for SGD with much weaker assumptions.

In this paper, instead of only considering the finite time horizon, we extend the classical idea of diffusion approximation to infinite time horizon, without assuming the boundedness of $\nabla f$. In our work, we study the traditional SGD with constant step size for general unbounded $\nabla f(\cdot, \xi)$, and show that SGD can be approximated in the weak sense by continuous-time SDEs in $\R^d$, by only assuming the strong convexity of the objective function $f$ and some other mild conditions. The SDE we use is different from \eqref{eq:classicalsde} in the sense that we have modified the  diffusion coefficient $\Sigma$.  Our approximation has second order weak accuracy and is uniform in time.  These will help us understand the discrete algorithms in the viewpoint of diffusion approximation and randomly perturbed dynamical system \cite{wu2018sgd}. With the diffusion approximation, it becomes possible that one is able to better understand the behavior of stochastic gradient noise in the SGD algorithm \cite{wu2021revisiting,simsekli2019tail}.

In particular, we study the long time behavior of $\left\{X_{n}\right\}_{n \geq 0}$ as $n$ approaches infinity in the flavor of backward error analysis of stochastic numerical schemes. After restricting $\Sigma$'s  support to some compact domain, where $\Sigma$ is the diffusion coefficient used in the classical method \eqref{eq:classicalsde}, we are then able to prove the following long time approximation:
\begin{gather}
\sup_{n\ge 0}\sup_{x\in B(0, R)}|U^n(x)-u(x, n\eta)|< C \eta^2,
\end{gather}
where $U$ is defined in \eqref{eq:Un} and $u$ is the solution to the backward Kolmogrov equation associated with the modified SDE. Compared with some previous work like \cite{feng2019uniform}, our result successfully weakens some of the assumptions, like the strong convexity of $f(\cdot;\xi)$ for every $\xi$ in the entire space $\R^d$, so that our result may be applied to more general objective functions.

The rest of the paper is organized as follows.  Before proving the main theorem, we establish some crucial auxiliary results in Section \ref{sec:setup}.
In Section \ref{sec:main}, we show that there is a uniform in time diffusion approximation for SGD with initial distribution on bounded sets by assuming strong convexity of the expected loss and some other mild regularity requirements. In Section \ref{sec:conclusion}, we perform discussion on the significance of long time diffusion approximation of SGD. We also discuss the case of general objective functions, for which the diffusion approximation on bounded set is valid up to $n\eta \sim O(\log(\eta^{-1}))$.

\section{Setup and auxiliary results}\label{sec:setup}

In this section, we first give some basic assumptions for our diffusion approximation, the most important of which is the strong convexity of the expected loss function. Next, we prove some auxiliary results useful for the diffusion approximation of SGD. In particular, Lemma \ref{lmm:Scontraction} ensures that SGD could not escape some compact set  after assuming some mild confinement conditions on the random loss functions; Lemma \ref{rmk3} ensures that the SDE solution we use to approximate SGD could not escape some compact set after modifying its diffusion coefficient; Proposition \ref{prop:derivativedecay} aims to estimate the high order derivatives of the associated Kolmogorov equation, which is crucial for the main theorem.

\begin{assumption}\label{ass:strongconvex}
The random loss functions and the expected loss satisfy the following conditions.
\begin{itemize} 
    \item[(i)] For any $\xi$, $f(\cdot, \xi)$ is smooth. For any compact set $K \subset \R^d$,  $\sup_{x\in K}\sup_{\xi}|\nabla f(x, \xi)|<\infty$. Moreover,  $f(\cdot,\xi)$ is confining in the sense that there exist $\nu>0$, $L >0$ independent of $\xi$ such that
    \[
    x\cdot\nabla f(x, \xi)\ge \nu |x|^2, \quad \forall |x|\ge L.
    \]
    
    \item[(ii)] $f(\cdot):=\mathbb{E}_{\xi}f(\cdot,\xi)$ is $\mu$-strongly convex in $\R^d$  i.e. $\forall x,y \in \R^d$,
    \[
    f(y) \geq f(x)+\nabla f(x) \cdot (y-x) + \frac{\mu}{2}|y-x|^2.
    \]
   
\end{itemize}
\end{assumption}

\begin{remark}
Here we are not assuming the growth rate of $f(\cdot; \xi)$. In later part of this paper, although we would have terms like its high order derivatives $\partial^{\alpha} f$, they are easy to control as the solutions to our SDE with a modified diffusion coefficient stay in a compact set. 
\end{remark}

\begin{remark}
The assumption on confinement of $f(\cdot;\xi)$ outside $B(0, L)$ is not restrictive because most models only use information in finite domains and this far-away behavior is satisfied by most models. 
\end{remark}

\begin{remark}
Note that we are only assuming the strong convexity of the expected loss $f(x)=\E f(x; \xi)$ instead of on each $f(\cdot; \xi)$. Though we are assuming the convexity of $f$ in the whole space, we actually only need the convexity of $f$ in a compact set where the SGD sees. Hence, our results in this paper actually apply to the behaviors near some local minimizers.
\end{remark}

With these assumptions, we then prove some auxilary results before our main theorem (Theorem \ref{thm:strongconvex}). The following says SGD will be trapped in a compact set if the initial measure is supported in a compact set.

\begin{lemma}\label{lmm:Scontraction}
Suppose Assumption \ref{ass:strongconvex} holds. Recall the definition of $X_n$ and the operator $\cS$ in Section \ref{intro}. Then we have the followings:
\begin{itemize}
    \item[(i)]  For the $L$ above in Assumption \ref{ass:strongconvex}, fix any $R>L$, then there exists $\eta_0>0$ such that for all $\eta\le \eta_0$, the condition $X_0\in B(0, R)$ implies that $X_n\in B(0, R)$ for all $n$.
    \item[(ii)]  For any continuous function $\phi$ and $x\in B(0, R)$, $(\cS \phi)(x)$ only depends on the values of $\phi$ in $B(0, R)$ with 
    \[\|\cS \phi\|_{L^{\infty}(B(0, R))} \le \|\phi\|_{L^{\infty}(B(0, R))}.\]
\end{itemize}
\end{lemma}
\begin{proof}
We set 
\[
M_1:=\sup_{\xi}\sup_{x\in B(0, L)}|\nabla f(x, \xi)|<\infty,
\]
 and 
\[
 M_2:=\sup_{\xi}\sup_{x\in B(0, R)}|\nabla f(x, \xi)|.
 \]
 We claim that we can set $\eta_0=\min \{(R-L)/M_1, 2\nu L^2/M_2^2\}$.
In fact, if $|X_0|\le L$, then $|X_1|=|X_0-\eta\nabla f(X_0, \xi_0)|\le L+\eta M_1\le R$. Otherwise, $L<|X_0|\le R$, using (i) in Assumption \ref{ass:strongconvex}, we have
\begin{gather*}
\begin{split}
|X_1|^2 &=|X_0|^2-2\eta X_0\cdot\nabla f(X_0,\xi_0)+\eta^2|\nabla f(X_0, \xi_0)|^2 \\
&\le |X_0|^2-2\eta\nu L^2+\eta^2M_2^2\le |X_0|^2.\\
\end{split}
\end{gather*}
Simple induction yields the first claim. The second claim regarding $\cS\phi$ is a straightforward corollary of the first one, using the definition $\cS\phi(x) = \mathbb{E}\phi(x - \eta \nabla f(x;\xi))$ in \eqref{eq:weakmaster} and the fact that $x \in B(0,R)$ implies $(x - \eta \nabla f(x;\xi)) \in B(0,R)$.

\end{proof}

In the following lemma, we show that if we modify the diffusion coefficient $\Sigma$ in \eqref{eq:classicalsde} outside a certain compact set, then the solution $X$ to the diffusion approximation, which is a modified version of \eqref{eq:classicalsde}, stays in some compact set. This then allows us to consider the $C^k$ norm of $g$ on a bounded domain for suitable function $g$ in later sections.

We define the modified diffusion coefficient $\Lambda$ as follows:

\begin{gather}\label{eq:defLambda}
    \Lambda  = 
    \begin{cases}
    \begin{aligned}
    & \Sigma,  \quad |x|\leq R,\\
    & 0,  \quad |x|> R_2\\
    & \text{is~smooth}, \quad R\leq |x| \leq R_2.
    \end{aligned}
    \end{cases}
\end{gather}
Moreover, we require $\Lambda$ to be positive semidefinite everywhere. This is clearly possible. Indeed, since $\Sigma$ is semidefinite, we can consider $\tilde{\sigma}$ being the smoothness modification of $\sqrt{\Sigma}$, so $\tilde{\sigma}\tilde{\sigma}^T$ is the smoothness modification of $\Sigma$, which is obviously semidefinite everywhere.

\begin{lemma}\label{rmk3}
Take $R>L$ in Lemma \ref{lmm:Scontraction}, and $\Lambda$ is chosen as above. Under Assumption \ref{ass:strongconvex}, for any initial value $x \in B(0,R)$, there exists $\eta_1>0$ such that for all $\eta \leq \eta_1$, the solution $X$ to the following SDE
\begin{gather}\label{eq:modifiedSDE}
\begin{split}
& dX=-\left[\nabla f(X)+\eta \left(\frac{1}{4}\nabla | \nabla f(X)|^2\right)\right]\,dt+\sqrt{\eta\Lambda(X)}\,dW,\\
& X(0; x)=x
\end{split}
\end{gather}
satisfies that for all $t\ge 0$,
\[
X(t;x) \in  \overline{B(0,R_2)},\quad a.s. \,.
\]
\end{lemma}

\begin{proof}
To show this, we make use of the classical Stroock-Varadhan support theorem\cite{stroock2020support}. More precisely, for any $T>0$, consider the corresponding control problem
\begin{equation}\label{eq:controlproblem}
dX^v=-\left[\nabla f(X^v)+\eta \left(\frac{1}{4}\nabla | \nabla f(X^v)|^2\right)\right]\,dt+\sqrt{\eta\Lambda}v(t)\,dt, \quad X^v|_{t=0} = x,
\end{equation}
with $v(\cdot) \in V:= C([0, T]; \R^d)$.

Denote by $S_x^T$ the support of $X_t$ in $C([0, T]; \R^d)$ under the topology induced by the uniform convergence norm $\|X\|:=\sup_{0\le t\le T}|X_t|$ ($X_t$ is the solution of the SDE \eqref{eq:modifiedSDE} at time $t$), and $C_x^T(V)$ the set of all solutions  of the ODE~\eqref{eq:controlproblem} when the function $v$ varies in $V$. The Stroock-Varadhan support theorem says that
\begin{gather}\label{eq:SVclosure}
    S_x^T = \overline{C^T_x(V)}.
\end{gather}

Next, we show that the support of the ODE solution $X^v(t)$ lies in $B(0,R_2)$ for all $t\le T$.
By Assumption \ref{ass:strongconvex}, when $\eta$ is small, one has 
\[
-x\cdot \left(\eta\left(\frac{1}{4}\nabla | \nabla f(x)|^2\right) + \nabla f(x)\right) < 0
\]
for $|x| = R_2$.

If $|X^v|$ ever reaches $R_2$, then
\begin{equation}
\frac{d}{dt}|X^v|^2 = 2X^v\cdot \dot{X^v} = -2X^v\cdot \left(\eta\left(\frac{1}{4}\nabla | \nabla f(X^v)|^2\right) + \nabla f(X^v)\right) < 0.
\end{equation}
This in fact implies that $|X^v|<R_2$ for all $t \leq T$. 

Finally, combining with \eqref{eq:SVclosure}, and since $T$ is arbitrary, we conclude that $\mathrm{supp} X_t \subset \overline{ B(0,R_2)}$ for all $t$.
\end{proof}

Without loss of generality, in the remaining part of this paper, we set $\eta_0=\eta_1$ for the convenience.

The following result is crucial for the long time approximation. Note that the diffusion matrix has been modified compared with that in classical results, as is stated in \eqref{eq:backwardKol}. More precisely, we consider the following Kolmogorov equation associated with the modified diffusion approximation \eqref{eq:modifiedSDE}:
\begin{gather}\label{eq:modifiedpde}
u_t=-\left(\nabla f+\eta \left(\frac{1}{4}\nabla | \nabla f|^2\right)\right)\cdot\nabla u+\frac{1}{2}\eta \Lambda:\nabla^2u,\quad u|_{t=0} = \varphi,
\end{gather}
where the diffusion matrix $\Lambda$ is defined in Lemma \ref{rmk3}. In the next proposition, we estimate the high order derivatives of its solution $u$.

Below, for a multi-index $J = (J_1,J_2,...,J_d)$, we denote $|J| := \sum_{i=1}^d J_i$, and $\partial^J  := \partial^{J_1}_1\partial^{J_2}_2...\,\partial^{J_d}_d$.

\begin{proposition}\label{prop:derivativedecay}
Let $u$ be the unique solution of the Kolmogorov equation \eqref{eq:modifiedpde}.  Assume the  initial data $\varphi \in C^{k}$. Suppose that Assumption \ref{ass:strongconvex} holds.  Then for each multi-index $J$ with $0<|J|\le k$,  there exist $\eta_0>0$,
$C_J > 0$, $\gamma_J>0$, and an integer $p_J > 0$  such that for all $\eta\le \eta_0$, $x\in B(0,R) \subset \R^d$,
\begin{gather}\label{eq:gradbd}
|\partial^J u(x, t)| \leq C_J (1+|x|^{p_J})e^{-\gamma_J t}.
\end{gather}
\end{proposition}
\begin{proof}

First of all, we consider $X(t; x)$ which solves the SDE \eqref{eq:modifiedSDE}. Then \eqref{eq:modifiedpde} is its associated backward Kolmogorov equation, and
\begin{gather}\label{eq:representationu}
u(x, t)=\mathbb{E}\varphi\left(X(t; x)\right), \quad \forall x \in B(0,R).
\end{gather}
For the convenience of notation, we denote
\begin{gather}
    \sigma := \sqrt{\Lambda / 2}.
\end{gather}

{\bf Step 1:} Estimates of $\mathbb{E}|X(t,x)|^{2m}$.

We claim that for nonnegative integer $m$,
\begin{gather}\label{eq:mattingly}
\mathbb{E}|X(t; x)|^{2m}\le C_m\left(1+|x|^{2m}e^{-\gamma_m t}\right).
\end{gather}
This can be proved easily using It\^o's formula and induction on $m$. For the convenience of notations, we will use $X$ to represent $X(t; x)$ in the current proof. Applying It\^o's formula to $|X|^{2m}$, for $m \geq 1$, we have
\begin{multline}\label{29}
d|X|^{2m}=\Big(2m|X|^{2m-2}X\cdot[-(\nabla f(X)+\eta (\frac{1}{4}\nabla | \nabla f(X)|^2))] \\+\frac{1}{2}(2m)|X|^{2m-2} M : 2\eta \sigma^2(X)\Big)\,dt +2m|X|^{2m-2}X\cdot \sqrt{2\eta}\sigma(X)\cdot dW,
\end{multline}
with
\[
    M:=I_d+(2m-2)\frac{X\otimes X}{|X|^2}.
\]
Taking expectation, we know that for any fixed $\bar{\mu} \in (0,\mu)$, it holds
\begin{gather}\label{210}
\frac{d}{dt}\mathbb{E}|X|^{2m}\leq -2m\bar{\mu}\mathbb{E}|X|^{2m}+A_2\mathbb{E}|X|^{2m-2},\quad m \geq 1.
\end{gather}
In the inequality above, $A_2$ is a positive constant depending on $m$.  Indeed, since $f$ is strongly convex, $X \cdot \nabla f(X) = X \cdot \left(\nabla f(X) - \nabla f(0)\right) + X \cdot \nabla f(0) \geq \mu |X|^2 + X \cdot \nabla f(0) $. Also, by Proposition \ref{rmk3}, $\mathbb{P}[X \in B(0,R_2)] = 1$. Clearly, the $C_0(B(0,R_2))$-norm of $\sigma$ and $\frac{1}{4}\nabla|\nabla f(\cdot)|^2$ is finite, so it holds that
\begin{gather*}
    \begin{aligned}
        \frac{d}{dt}\mathbb{E}|X|^{2m} &\leq -2m\mu \mathbb{E}|X|^{2m} + A_0 \mathbb{E} |X|^{2m-2} + \mathbb{E}\left[(2m|X|^{2m-2}X\cdot\left( \eta \left(\frac{1}{4}\nabla|\nabla f(X)|^2\right) - \nabla f(0)\right)\right]\\
        & \leq -2m\mu\mathbb{E}|X|^{2m} + A_0 \mathbb{E}|X|^{2m-2} + A_1 \mathbb{E}|X|^{2m-1}\\
        & = -2m\mu \mathbb{E}|X|^{2m} + A_0 \mathbb{E}|X|^{2m-2} + A_1 \mathbb{E}\sqrt{(\epsilon_m |X|^{2m})(\frac{1}{\epsilon_m}|X|^{2m-2})}\\
        & \leq -2m\mu \mathbb{E}|X|^{2m} + A_0 \mathbb{E}|X|^{2m-2} + \frac{1}{2}A_1 \mathbb{E}\left[\epsilon_m |X|^{2m}+\frac{1}{\epsilon_m}|X|^{2m-2}\right]\\
        & \leq -2m\bar{\mu} \mathbb{E}|X|^{2m} + A_2 \mathbb{E}|X|^{2m-2}.
    \end{aligned}
\end{gather*}
Above, we have chosen $\epsilon_m$ small enough such that $\frac{1}{2}A_1\epsilon_m < 2m(\mu-\bar{\mu})$ to ensure that the last inequality holds. So now \eqref{210} is obtained. 

Next, we consider induction on $m$. \eqref{eq:mattingly} is obvious for $m=0$. For $m>0$, using induction hypothesis, we have
\begin{gather}
\frac{d}{dt}\mathbb{E}|X|^{2m}\leq -2m\bar{\mu}\mathbb{E}|X|^{2m}+A_3(1+|x|^{2m-2}e^{-\gamma_{m-1}t}),
\end{gather}
where $A_3$ is a positive constant depending on $m$. Using Gr\"ownwall's inequality, we have 
\begin{gather*}
\begin{aligned}
\mathbb{E}|X|^{2m} & \leq e^{-2m\bar{\mu}t}|x|^{2m} +\int_0^t A_3(1+|x|^{2m-2}e^{-\gamma_{m-1}s}) e^{-2m\bar{\mu}(t-s)}ds\\
&\leq c_m(1+|x|^{2m}e^{-\gamma_mt}),
\end{aligned}
\end{gather*}
for some positive constants  $A_4$, $A_5$, $c_m$ and $\gamma_m$, and the last inequality is due to Young's inequality. Hence \eqref{eq:mattingly} holds for any nonnegative integer $m$.

{\bf Step 2:} Estimates of the moments of $\partial_x^J X(t,x)$

It is well-known that the stochastic map $x\mapsto X(t, x)$ is a diffeomorphism almost surely for all $t$ \cite{kunita1997stochastic,le1984stochastic}, so it is valid here to take partial derivatives with respect to $x$. Below, we will consider
\[
X^{(J)}(t, x):=\partial_x^{J}X(t,x).
\]
Similarly, we will use $\partial^J$ to represent $\partial_x^J$ and
$X^{(J)}$ to represent $X^{(J)}(t, x)$ for convenience. 

If $|J|=1$ (note that $|J| = \sum_{i=1}^dJ_i$ ), by similar discussion in \cite{elworthy1994formulae},  $X^{(J)}$ satisfies the following SDE
\begin{equation}
\begin{aligned}
 & dX^{(J)}  =-\left[\nabla^2f(X)+\eta\left(\nabla (\frac{1}{4}\nabla | \nabla f|^2)\right)^T\right]\cdot X^{(J)}dt+\sqrt{2\eta}(X^{(J)}\cdot\nabla\sigma)\cdot dW,\\
& X^{(J)}(0; x)=e_J.
\end{aligned}
\end{equation}
Formally, the equation for $X^{(J)}$ is obtained by taking derivative of $X$ on $x$ in the SDE \eqref{eq:modifiedSDE}. Obviously, $X^{(J)}$ also has compact support, though we do not use this property in our proof.

Applying It\^o's formula to $|X^{(J)}(t; x)|^p$,  for $p\ge 2$
\begin{gather*}
\begin{aligned}
 d|X^{(J)}|^{p}  =&\Big[p|X^{(J)}|^{p-2}X^{(J)}\cdot\left(-\left(\nabla^2f(X)+\eta\frac{1}{4}\nabla^2 | \nabla f|^2 \right)\right)\cdot X^{(J)}\\
& +\eta p|X^{J}|^{p-2}X^{(J)}_kX_{\ell}^{(J)}\partial_{k}\sigma_{i,\cdot}\partial_{\ell}\sigma_{i,\cdot}:M_J\Big]dt\\
& +p|X^{(J)}|^{p-2}X^{(J)}\cdot\sqrt{2\eta}(X^{(J)}\cdot\nabla\sigma) \cdot dW
\end{aligned}
\end{gather*}
with 
\[
 |X^{(J)}(0,x)|^{p}=|e_{J}|^{p}=1.
\]
Above, 
\[
M_J:=I_d+(p-2)\frac{X^{(J)}\otimes X^{(J)}}{|X^{J}|^2}.
\]
Similarly with \eqref{210}, using the fact that $X$ is bounded, for all $\eta$ small enough, the matrix $\left(\nabla^2f+\eta\frac{1}{4}\nabla^2 | \nabla f|^2 + \eta \partial \sigma_{i,\cdot}\partial \sigma_{i,\cdot} :M_d\right)$ is positive definite, after taking expectation, it holds
\begin{gather}
\frac{d}{dt}\mathbb{E}|X^{(J)}|^p\leq -p \gamma \mathbb{E}|X^{(J)}|^{p}.
\end{gather}
Here, $\gamma$ is a positive constant. By Gr\"onwall's inequality,

\begin{gather}\label{eq:j=1}
\mathbb{E}|X^{(J)}(t; x)|^p \le \exp(-p\gamma t)|e_J|^p=\exp(-p\gamma t).
\end{gather}

Now, we do induction for general $J$. Suppose we have constructed $X^{(I)}(t; x)$ with $|I|\le |J|-1$, such that the moments satisfy:
\begin{gather}\label{eq:momentcontrol}
\mathbb{E}|X^{(I)}(t; x)|^p \le C(1+|x|^{q_{I}})\exp(-\gamma_{p,I} t),~~p\ge 2.
\end{gather}

Now, we consider $J$. The new introduced variable $X^{(J)}$ satisfies the following equation:
\begin{multline}\label{eq:eqforhighQJ}
dX^{(J)}=-\left(\nabla^2f+\eta\left(\nabla \left(\frac{1}{4}\nabla | \nabla f|^2\right)\right)^T\right)\cdot X^{(J)} \,dt
+Q_J\left(\partial^{\alpha}f, \partial^{\beta}\left(\frac{1}{4}\nabla | \nabla f|^2\right), X^{(I)} \right)\,dt\\
+\sqrt{2\eta}\left(X^{(J)}\cdot\nabla\sigma+R_J\left(\partial^{\alpha}\sigma,  X^{(I)}\right)\right)\cdot dW,
\end{multline}
with the initial condition
\[
X^{(J)}(0;x)=0 \in \mathbb{R}^d.
\]
In equation \eqref{eq:eqforhighQJ}, $Q_J\left(\partial^{\alpha}f, \partial^{\beta}(\frac{1}{4}\nabla | \nabla f|^2), X^{(I)}\right)$ is a polynomial of $\partial^{\alpha}f$ with $|\alpha|\le |J|$,
$\partial^{\beta}(\frac{1}{4}\nabla | \nabla f|^2)$ with $|\beta|\le |J|$ , and $X^{(I)}$ with $|I|\le |J|-1$. Similarly, $R_J$ is a polynomial 
`of $\partial^{\alpha}\sigma$ with $|\alpha|\le |J|$ and $X^{(I)}$ with $|I|\le |J|-1$. Note that each term in both polynomials has some $X^{(I)}$ with positive order.

Again, using It\^o's formula, we find that for $p\ge 2$,
\begin{multline*}
\frac{d}{dt}\mathbb{E}|X^{(J)}|^p =\mathbb{E}p|X^{(J)}|^{p-2}X^{(J)}\cdot \Big[\left(-\nabla^2f-\eta\left(\nabla \left(\frac{1}{4}\nabla | \nabla f|^2\right)\right)^T\right)\cdot X^{(J)}\\
+Q_J\left(\partial^{\alpha}f, \partial^{\beta}\left(\frac{1}{4}\nabla | \nabla f|^2\right), X^{(I)} \right)\Big]
+\eta p\mathbb{E}|X^{J}|^{p-2}X^{(J)}_kX_{\ell}^{(J)}\partial_{k}\sigma_{i,\cdot}\partial_{\ell}\sigma_{i,\cdot}:M_J\\
+\eta p\mathbb{E}|X^{J}|^{p-2} R_JR_J^T:M_J
+2\eta p\mathbb{E}|X^{J}|^{p-2}(X_{J}\cdot \nabla \sigma)\cdot R_J^T:M_J.
\end{multline*}

Similarly with the case $|J|=1$, for $\eta < \eta_0$, the first and third term above can be bounded above by $-p\gamma\mathbb{E}|X^{(J)}|^p$ with $\gamma$ being a positive constant. Since $X$ is bounded, the second term (``$Q_J$" term) can be bounded above by $p\mathbb{E}|X^{(J)}|^{p-1}\sum_{0<|I|\le |J|-1}C_{1,I}|X^{(I)}|^{q_I}$. Other terms can be bounded similarly. So we have
\begin{multline*}
\frac{d}{dt}\mathbb{E}|X^{(J)}|^p\le -p\gamma\mathbb{E}|X^{(J)}|^p
+p\bar{A}\mathbb{E}|X^{(J)}|^{p-1}\sum_{0<|I|\le |J|-1}C_{1,I}|X^{(I)}|^{q_I}
\\+p\eta \mathbb{E}|X^{(J)}|^{p-2}\sum_{0<|I|\le |J|-1}C_{2,I}|X^{(I)}|^{r_{2,I}}.
\end{multline*}

Next, we estimate the $\mathbb{E}|X^{(J)}|^{p-1}$ term and the $\mathbb{E}|X^{(J)}|^{p-2}$ term.
Applying Young's inequality, for any $\delta>0$,  we have
\[
\mathbb{E}|X^{(J)}|^{p-1}\sum_{|I|\le |J|-1}C_{1,I}|X^{(I)}|^{q_I}
\le \delta\frac{(p-1)\mathbb{E}|X^{(J)}|^p}{p}+C_3\frac{1}{p\delta}\mathbb{E}\left(\sum_{|I|\le |J|-1}C_{1,I}|X^{(I)}|^{q_I}\right)^p.
\]
The $\mathbb{E}|X^{(J)}|^{p-2}$ term can be similarly controlled if $p>2$. If $p=2$, we just leave it as it appears. Now, we choose $\delta$ small enough such that $\gamma-2\delta>0$. Then, for $\eta$ small enough, combining with the induction assumption on the moments \eqref{eq:momentcontrol}, we find that
\begin{gather}
    \frac{d}{dt} \mathbb{E}|X^{(J)}|^p \leq -p\bar{\gamma} \mathbb{E}|X^{(J)}|^p +  C(1+|x|^{q_{J}})\exp(-\gamma_{p,J} t),
\end{gather}
where $\bar{\gamma}$, $C$ and $\gamma_{p,J}$ are positive constants. Hence \eqref{eq:momentcontrol} also holds for $J$ using Gr\"onwall's inequality. Namely, we can control the moments by
\begin{gather}\label{eq:momentcontrol-J}
\mathbb{E}|X^{(J)}(t; x)|^p \le C(1+|x|^{q_{J}})\exp(-\gamma_{p,J} t),~~p\ge 2.
\end{gather}

{\bf Step 3:} Estimates $\partial^J u(x, t)$.

Finally, using \eqref{eq:representationu} and similar discussion in \cite{elworthy1994formulae}, we have

\begin{gather}\label{eq:DerivativerepresentationGeneralJ}
    \partial^Ju(x, t)= 
    \begin{cases}
     \mathbb{E}\left[\nabla \varphi(X) \cdot X^{(J)}\right], & \quad |J| = 1,\\
     \mathbb{E}\left[\nabla\varphi(X)\cdot X^{(J)}+ P_J\left(\partial^{\alpha}\varphi(X), X^{(I)}\right) \right], & \quad |J| \geq 2,
    \end{cases}
\end{gather}
where $P_J(\varphi, X^{(I)})$ is a polynomial of $X^{(I)}$ with $|I|\le |J|-1$ and  $\partial^{\alpha}\varphi(X)$ with $|\alpha|\le |J|$. Using \eqref{eq:mattingly}, \eqref{eq:momentcontrol-J}, and applying H\"older inequality to \eqref{eq:DerivativerepresentationGeneralJ}  yields the result for $\partial^J u$. This then finishes the proof.
\end{proof}

\begin{remark}
The initial value $x$ is inside $B(0,R)$, but the solution $X$ to the SDE \eqref{eq:modifiedSDE} can be outside the ball $B(0, R)$. However, this proposition ensures that $\sup_{x\in B(0,R)}|\partial^J u(x,t)|$ can still be controlled.
\end{remark}

\section{Main theorem: uniform in time diffusion approximation}\label{sec:main}

Now, we fix $R>L$ mentioned in Lemma \ref{lmm:Scontraction} and consider the initial distribution/law of $X_0$ that is supported in $B(0, R)$.
We consider $\Lambda$ as in \eqref{eq:defLambda}. Observe that under such setting, it holds that
\begin{gather}
\|\Lambda\|_{C^k(\mathbb{R}^d)}\le C_k\|\Sigma\|_{C^k(B(0, R))}.
\end{gather}

Now consider the SDE
\begin{gather}\label{eq:newSDE}
dX=-\left(\nabla f(X)+\eta \left(\frac{1}{4}\nabla|\nabla f(X)|^2\right) \right)\,dt+\sqrt{\eta \Lambda(X)}\,dW,\quad X|_{t=0} = x.
\end{gather}
Let $u$ be the solution to the Kolmogorov equation $\partial_tu=\cL u$ with $u(x, 0)=\varphi(x)$. (Recall the definition of $\cL$ and $u$ in \eqref{eq:modifiedpde}.)
We have the following long time diffusion approximation.
\begin{theorem}\label{thm:strongconvex}
Suppose Assumption \ref{ass:strongconvex} holds. Then the SDE \eqref{eq:newSDE} approximates the SGD \eqref{eq:sgd} with weak second order accuracy uniformly in time for initial distributions supported in $B(0,R)$. More precisely,  for $R$ that is required in Lemma \ref{lmm:Scontraction} and $\varphi\in C^{4}$, there exists $C=C(\varphi,R)>0$ that depends on $\varphi$ and $R$ but independent of $\eta$ such that when $\eta$ is sufficiently small (recall \eqref{eq:Un} for $U^n$)
\begin{gather}
\sup_{n\ge 0}\sup_{x\in B(0, R)}|U^n(x)-u(x, n\eta)|< C \eta^2.
\end{gather}
\end{theorem}

Now, we do some preparation to prove this theorem. For the convenience of the presentation, we introduce
\begin{gather}
u^n(x):=u(x, n\eta),
\end{gather}
and denote the generator associated with \eqref{eq:newSDE}:
\begin{gather}
\cL=-\left(\nabla f(x)+\eta (\frac{1}{4}\nabla | \nabla f(x)|^2)\right)\cdot\nabla+\frac{1}{2}\eta\Lambda(x):\nabla^2=:\cL_1+\eta\cL_2,
\end{gather}
so that $\cL_1=-\nabla f(x)\cdot\nabla$ and $\cL_2=-\left(\frac{1}{4}\nabla | \nabla f(x)|^2\right)\cdot\nabla+\frac{1}{2}\Lambda:\nabla^2$. Since $u^{n+1}=e^{\eta \cL}u^n$, direct semigroup expansion would give us that
\begin{gather}
\begin{split}
& u^j=u^{j-1}+\eta \cL u^{j-1}(x)+\frac{1}{2}\eta^2 \cL^2u^{j-1}(x)+\eta^3 R_1\\
& =u^{j-1}+\eta\left(-\nabla f-\eta (\frac{1}{4}\nabla | \nabla f|^2)\right)\cdot\nabla u^{j-1} +\frac{\eta^2}{2}\Lambda:\nabla^2u^{j-1}
+\frac{\eta^2}{2} \nabla f\cdot\nabla(\nabla f\cdot\nabla u^{j-1})+\eta^3 R_2\\
& = u^{j-1} -\eta \nabla f \cdot \nabla u^{j-1} + \frac{\eta^2}{2} (\Lambda + \nabla f \otimes \nabla f) : \nabla^2 u^{j-1} + \eta^3 R_2.
\end{split}
\end{gather}
In the remainder term $R_2$, we have derivatives of $u$ and $f$. It seems that we need the $C^6$ norms of $u$ and $f$ to bound the terms of order $\eta^3$. In fact, we can relax this.
\begin{lemma}\label{lmm:localtruncation}
For $R>0$ that is required in Lemma \ref{lmm:Scontraction}, then it holds that
\begin{multline}
\sup_{x\in B(0,R)}\left|u^{n+1}(x)-\left(u^n+\eta \cL u^n+\frac{1}{2}\eta^2\nabla f\cdot\nabla(\nabla f\cdot\nabla u^n)\right) \right| \\
\le C(\|f\|_{C^4(B(0, R))})\sup_{t\in [t^n, t^{n+1}]}\|u(\cdot,t)\|_{C^{1,4}(B(0,R))}\eta^3,
\end{multline}
where $\|u\|_{C^{p,q}(U)}:=\sum_{p\le |\alpha|\le q}\sup_{x\in U}|\partial^{\alpha}u|$ and $C(\|f\|_{C^4(B(0, R))})$ is a constant depending on $\|f\|_{C^4(B(0, R))}$.
\end{lemma}
\begin{proof}
The proof is straightforward using the equivalent integral representation:
\begin{gather}
u(x,t)=u^n(x)+\int_{t^n}^{t}\cL_1u(x,s)+\eta \cL_2u(x,s)\,ds.
\end{gather}
If we directly do semigroup expansion, we will have $\cL^3 u$ terms in the remainder, which require $C^6$ norms of $u$, so we  consider using the integral form to solve this. Our strategy is to put the equivalent integral representation of $u$ into the integral on the right side of the formula above. We then repeat this process until every term in the residue is no less than $\eta^3$. By doing so, we can avoid the $C^6$ norms of $u$. Simple calculation gives
\begin{multline}
u(x,t)=u^n(x)+\int_{t^n}^{t}\cL_1\left(u^n(x)+\int_{t^n}^{s}\cL_1u(x,\tau)+\eta \cL_2u(x,\tau)d\tau\right)ds
\\+\eta\int_{t^n}^{t}\cL_2\left(u^n(x)+\int_{t^n}^{s}\cL_1u(x,\tau)+\eta\cL_2u(x,\tau)d\tau\right)ds
\\=u^n(x)+(t-t^n)(\cL_1+\eta\cL_2)u^n(x)+\int_{t^n}^{t}\int_{t^n}^{s}\cL_1^2u(x,\tau)d\tau ds
+\eta \int_{t^n}^{t}\int_{t^n}^{s}\cL_1\cL_2u(x,\tau)d\tau ds\\+\eta\int_{t^n}^{t}\int_{t^n}^{s}\cL_2\cL_1u(x,\tau)d\tau ds
+\eta^2\int_{t^n}^{t}\int_{t^n}^{s}\cL_2^2u(x,\tau)d\tau ds
\\=u^n(x)+(t-t^n)(\cL_1+\eta\cL_2)u^n(x)+\int_{t^n}^{t}\int_{t^n}^{s}\cL_1^2\left(u^n(x)+\int_{t^n}^{\tau}\cL_1u(x,z)+\eta\cL_2u(x,z)dz\right)d\tau ds
\\+\eta \int_{t^n}^{t}\int_{t^n}^{s}\cL_1\cL_2u(x,\tau)d\tau ds+\eta\int_{t^n}^{t}\int_{t^n}^{s}\cL_2\cL_1u(x,\tau)d\tau ds
+\eta^2\int_{t^n}^{t}\int_{t^n}^{s}\cL_2^2u(x,\tau)d\tau ds
\\=u^n(x)+(t-t^n)(\cL_1+\eta\cL_2)u^n(x)+\frac{1}{2}(t-t^n)^2\cL_1^2u^n(x)+\int_{t^n}^{t}\int_{t^n}^{s}\int_{t^n}^{\tau}\cL_1u(x,z)+\eta\cL_2u(x,z)dzd\tau ds
\\+\eta \int_{t^n}^{t}\int_{t^n}^{s}\cL_1\cL_2u(x,\tau)d\tau ds+\eta\int_{t^n}^{t}\int_{t^n}^{s}\cL_2\cL_1u(x,\tau)d\tau ds
+\eta^2\int_{t^n}^{t}\int_{t^n}^{s}\cL_2^2u(x,\tau)d\tau ds.
\end{multline}
Setting $t$ at $t^{n+1}$, since the initial value $x  \in B(0,R)$, and
\[
|\cL_i^{\alpha}\cL_j^{\beta}u|\le C(\|f\|_{C^{i\alpha+j\beta}(B(0,R))})\|u\|_{C^{1,i\alpha+j\beta}(B(0,R))} 
\]
with $i,j \in \{1,2\}$ and  $i\alpha+j\beta \leq 4$, the claim follows.
\end{proof}


Now all the preparation work has been done. Then, we can prove our main theorem.

\begin{proof}[Proof of Theorem \ref{thm:strongconvex}]

Since $U^n=\mathcal{S}^n\varphi=\mathcal{S}^nu(\cdot, 0)$, we have 
\begin{gather*}
U^n(x)-u(x, n\eta)=\sum_{j=1}^n \mathcal{S}^{n-j}(\mathcal{S}u^{j-1}-u^j)(x).
\end{gather*}
By Lemma \ref{lmm:Scontraction} , we have $\|U^n(x)-u(x, n\eta)\|_{L^{\infty}(B(0, R))}\le \sum_{j=1}^n \|Su^{j-1}-u^j\|_{L^{\infty}(B(0, R))}$.
Now, direct Taylor expansion shows that
\begin{multline*}
(\mathcal{S}u^{j-1})(x)=\mathbb{E}u^{j-1}\left(x-\eta\nabla f(x;\xi)\right)
=u^{j-1}(x)-\eta\nabla f(x)\cdot\nabla u^{j-1}(x)\\
+\frac{1}{2}\eta^2\left(\Sigma+\nabla f(x)\otimes \nabla f(x)\right):\nabla^2u^{j-1}(x)
+\eta^3 R,
\end{multline*}
where $\|R\|_{L^{\infty}(B(0, R))}\le C(\|f(x,\xi)\|_{C^1(0,R)}) \|u^{j-1}\|_{C^3(B(0, R))}$ by Taylor expansion.
Note that $\Sigma+\nabla f(x)\otimes \nabla f(x)=\mathbb{E}\nabla f(x,\xi)\otimes \nabla f(x,\xi)$.

Note that in $B(0, R)$, $\Lambda=\Sigma$.
By  Lemma \ref{lmm:localtruncation}, there exists a constant $C$ depending on $\sup_{\xi}\|f(\cdot, \xi)\|_{C^4}$ such that
\[
\|\mathcal{S}u^{j-1}-u^j\|_{L^{\infty}(B(0, R))} \le C(\|f(\cdot)\|_{C^4(B(0, R))}) \sup_{t\in [t^{j-1}, t^j]}\|u(\cdot, t)\|_{C^{1,4}(B(0, R))} \eta^3.
\]
By Proposition \ref{prop:derivativedecay}, there exists $\beta>0$ such that for $\eta$ is sufficiently small it holds
\[
\sum_{j=1}^n\|\mathcal{S}u^{j-1}-u^j\|_{L^{\infty}(B(0, R))}  \le C(\|f(\cdot)\|_{C^4})  \eta^3 \sum_{j}C_4 e^{-\beta j \eta}
\le C\eta^2.
\]
The claim therefore follows.
\end{proof}

\section{Discussions}\label{sec:conclusion}
In this paper we extended the classical diffusion approximation for SGD from finite time to infinite time, provided that the expected loss function is strongly convex. Here, we perform some illustrating discussion.

\subsection{Significance of the diffusion approximation}

Usual SGD diffusion approximation is on finite time. The extension to a uniform-in-time approximation in this work would allow us to analyze the aysmptotic behavior of SGD using the tools from SDEs with small noise, like those about random perturbed dynamical systems \cite{freidlin2004random,MR2571413}.

As a first example, in \cite{feng2019uniform}, by assuming that each $f(\cdot;\xi)$ is convex, the SGD has been shown to have an invariant measure. Without the convexity assumption on each $f(\cdot, \xi)$ as in our current work, it is not straightforward to study the ergodicity of SGD directly. The uniform $O(\eta^2)$ weak approximation, however, provides a possible way to investigate the long time behavior under these weaker conditions.
In fact,  if the expected loss $f$ is strongly convex within the region where the SGD sees, we have the uniform-in-time approximation and the diffusion approximation SDE can be shown to have an exponential ergodicity due to the convexity of $f$. The uniform-in-time approximation ensures that the distribution of SGD is only $O(\eta^2)$ away from this invariant measure, which then tells us the long time behavior.

For another example, the uniform-in-time diffusion approximation may enable us to investigate the behavior of SGD near local minimizers, by analyzing the large deviation behaviors of  \eqref{eq:modifiedSDE}, which was done in  \cite{hulililiu2018} where the uniform-in-time approximation was taken for granted. Moreover, as in the work of Samuel \cite{smith2020generalization,smith2021origin},  the term $\frac{1}{4}\eta |\nabla f|^2$ may be regarded as the implicit regularizer to the loss landscape. Hence, with the uniform-in-time SDE approximation, we may perform similar study the behavior near the local minimizers and also how the regularizer affects the behaviors near the local minimizers and thus possibly the generalization ability.

\subsection{Discussion on nonconvex case}

For general loss function, given initial value in $B(0, R)$, SGD can hit the boundary of $B(0, 2R)$ in $M\eta^{-1}$ steps, where $M$ depends on the $L^{\infty}$ norms of $f(\cdot, \xi)$ on $B(0, 2R)$. Similarly as in Section \ref{sec:setup}, we can modify the values $\Sigma$ outside $B(0, 2R)$ so that it is a smooth function with compact support. Then, it is possible to show that for $x\in B(0, R)$
\[
\mathbb{E}|X^{J}|^p\le C \exp(\gamma_{I,p} t),
\]
where $\gamma_{I,p}$ depends on the values of $f$ in $B(0, 2R)$. Using this, one can show that
\[
\|\partial^{J}u\|_{L^{\infty}(B(0, R))}\le C\exp(\alpha t).
\]
With similar computation, we find the $O(\eta^2)$ diffusion approximation is valid up to time $T\sim \log(1/\eta)$. Actually, we can observe this by simply replacing $\exp(-\beta j \eta)$ with $\exp(\alpha j \eta)$ in the last step of the proof of theorem \ref{thm:strongconvex}.
\begin{proposition}\label{prop:2}
Assume the initial point of SGD is chosen from $B(0, R)$ for some $R>0$. For any $\varphi\in C^{\infty}$, there exists $\beta>0$, $C>0$ that depends on $\varphi$ and the norms of $f(\cdot,\xi)$ in $B(0, 2R)$ such that 
\begin{gather}
\sup_{n\eta\le \beta\ln(1/\eta)}\sup_{x\in B(0, R)}|U^n(x)-u(x, n\eta)|\le C \eta^2.
\end{gather}
\end{proposition}
In applications, the expected loss functions are generally not strongly convex in the focused region. In particular, the problem of investigating the behavior of SGD near saddle points are important for understanding some special behaviors in SGD \cite{kleinberg2018alternative}.
As mentioned in \cite{hulililiu2018},  Kifer proved that the SDE 
\[
dX=-\nabla f(X)dt+\sqrt{\epsilon}\sigma\, dW
\]
escapes the saddle point of $f$ in $O(\log(\e^{-1}))$ time. Using the diffusion approximation, we expect that SGD escapes the saddle point in a typical steps of order $O(\eta^{-1}|\log \eta|)$, which is the direct result of Proposition \ref{prop:2}. Since the diffusion approximation is valid exactly up to this time regime,  it is not realistic to use diffusion approximation to justify this guess. We leave this problem for future.

\subsection{Extensions and possible future work}
In this paper, we established the uniform-in-time diffusion approximation of SGD for the strongly convex case, but extensions to general non-convex case is really difficult. This is mainly due to the fact that the diffusion coefficient of SGD is usually of $O(\sqrt{\eta})$, where $\eta$ is the footstep or learning rate. In Stochastic Gradient Langevin Dynamics (SGLD), however, the diffusion coefficient is of $O(1)$ \cite{welling2011bayesian}. Hence in the future it is possible to prove the high order uniform-in-time diffusion approximation of SGLD with non-convex potentials.

\section*{Acknowledgement}

This work is financially supported by the National Key R\&D Program of China, Project Number 2021YFA1002800. The work of L. Li was partially supported by Shanghai Municipal Science and Technology Major Project 2021SHZDZX0102, NSFC 11901389 and 12031013, and Shanghai Sailing Program 19YF1421300. We would like to thank Yang Jing for the help of some formula derivation.

\bibliographystyle{unsrt}
\bibliography{main}

\end{document}